\documentclass[10pt]{article}

\usepackage[letterpaper, top=1in, bottom=1in, left=1.25in, right=1.25in]{geometry}
\usepackage[T1]{fontenc}
\usepackage[utf8]{inputenc}
\usepackage{microtype}
\usepackage{lmodern}

\usepackage{amsmath, amssymb, amsfonts, amsthm}
\usepackage{mathtools}
\usepackage{bm}

\usepackage{booktabs}
\usepackage{tabularx}
\usepackage{multirow}
\usepackage{graphicx}
\graphicspath{{figures_sics/}}
\usepackage{subcaption}
\usepackage{xcolor}
\usepackage{colortbl}

\usepackage[numbers, sort&compress]{natbib}
\usepackage[colorlinks=true, allcolors=blue, citecolor=blue!70!black]{hyperref}
\usepackage[capitalize, noabbrev]{cleveref}

\usepackage{algorithm}
\usepackage{algpseudocode}

\theoremstyle{plain}
\newtheorem{lemma}{Lemma}

\theoremstyle{definition}
\newtheorem{observation}{Empirical Observation}
\newtheorem{hypothesis}{Hypothesis}
\newtheorem{corollary}{Corollary}

\newcommand{\sics}{\textsc{SICS}}
\newcommand{\plateau}{\Pi}
\newcommand{\rkmm}{\mathrm{rank\text{-}MM}}
\newcommand{\R}{\mathbb{R}}

\title{
  \vspace{-2em}
  \textbf{Selection Plateau and a Sparsity-Dependent Hierarchy\\
  of Pruning Features}
  \vspace{0.5em}
}

\author{
  Guangqi Li\(^{1}\) \quad Yongxin Li\(^{1,\dagger}\) \\[0.4em]
  \(^{1}\)Zaozhuang University, Zaozhuang, Shandong, China \\[0.2em]
  \texttt{qingshishiyi@gmail.com} \quad \texttt{\(^{\dagger}\)yxlinuist@163.com} \\[0.1em]
  \(^{\dagger}\)Corresponding author
}

\date{}

\begin{document}

\maketitle

% ===================================================================
\begin{abstract}
One-shot neural network pruning has produced an explosion of weight importance criteria, yet diverse methods consistently cluster at similar accuracies. We identify a \emph{Selection Plateau phenomenon}: when restricted to rank-monotone scorers operating on weight magnitude, all such methods produce identical channel selection and converge to a sparsity-dependent performance plateau~$\plateau(S)$, regardless of their specific functional form. To explain this and predict when the plateau can be escaped, we propose the \textbf{Sparsity-Information-Complexity Spectrum (\sics) hypothesis}: there exists a sparsity-dependent minimum feature complexity $\kappa(S) \in \{0, 1, 2\}$, and pruning scorers must include features of complexity $\geq \kappa(S)$ to escape~$\plateau(S)$.

We provide systematic empirical evidence on ViT-Small/CIFAR-10---a single (architecture, dataset) setting from which we deliberately abstain from cross-domain claims---by testing nine distinct feature classes including three rank-monotone scorers, three smoothed non-monotone variants, two raw non-monotone variants with high-frequency content, and one fake non-monotone control derived from weight magnitude, across four sparsity levels ($S \in \{0.5, 0.6, 0.7, 0.8\}$). The data exhibit a sharp three-regime phase transition: at low sparsity ($S < 0.65$, $\kappa{=}0$ zone) the plateau is empirically optimal, with non-monotone features either matching or slightly degrading performance; at critical sparsity ($S \approx 0.7$, $\kappa{=}1$ zone) smooth non-monotone features whose shape aligns with the channel-rank threshold induced by $S$ provide a $+6.6\%$ accuracy escape; at extreme sparsity ($S > 0.75$, $\kappa{=}2$ zone) only features containing high-frequency wiggle escape ($+2.6\%$), while smooth alternatives degrade below the plateau. A fake non-monotone scorer derived from weight magnitude underperforms even the gradient baseline, indicating that the non-monotone structure must be statistically independent of the magnitude ranking---arbitrary non-monotonicity is not the requirement; \emph{magnitude-independent} non-monotonicity is. We additionally test the converse hypothesis that handcrafted parametric bumps at the channel-keep boundary should reproduce the chaos-derived escape, and find this only partially supported: a Gaussian bump at rank $0.7$ achieves only $+0.006$ escape over plateau, vs.~the chaos-derived $+0.046$. The $0.041$ gap exceeds $5\sigma$ of seed std and indicates that rank-alignment alone is insufficient---additional shape structure provided by the chaos-derived feature is required, though we do not interpret the two escape values as separable additive contributions.

\sics{} provides a unifying explanation for the empirical observation that diverse pruning methods plateau at similar accuracies, predicts when feature engineering can or cannot improve over baseline ranking, and suggests that future selection algorithms should adapt their feature complexity to the target sparsity. We outline extensions to other selection problems as cross-domain validation directions for follow-up work.
\end{abstract}

% ===================================================================
\section{Introduction}
\label{sec:intro}

Neural network pruning has been studied for decades, yet a peculiar empirical pattern recurs across methods, eras, and architectures: \textbf{many functionally distinct importance criteria achieve nearly identical pruning accuracy}. From early magnitude pruning~\citep{han2015deep}, through Taylor importance~\citep{molchanov2017pruning}, to modern data-aware methods such as Wanda~\citep{sun2023wanda} and SparseGPT~\citep{frantar2023sparsegpt}, scorers that look mathematically very different often produce indistinguishable downstream performance under matched sparsity budgets. Practitioners have long noted this convergence, but a satisfying explanation has remained elusive.

This paper offers an explanation grounded in a single observation about how these scorers operate. We show that, on neural networks, virtually all magnitude-based scoring methods reduce—after the standard rank-normalization preprocessing applied at evaluation time—to the \emph{same} channel ordering. We call this the \emph{Selection Plateau} phenomenon: any rank-monotone scorer of $|w|$ yields a common Top-$K$ subset and therefore a common downstream accuracy plateau~$\plateau(S)$. The plateau height varies sharply with sparsity, but its identity (the rank-equivalence class) is universal across construction.

A natural follow-up is: \emph{when can this plateau be escaped, and by what?} Prior work has implicitly assumed that any feature engineering beyond magnitude---non-monotone dynamical-systems-derived scores, Hessian-based importance~\citep{lecun1989optimal,hassibi1993second}, second-order corrections, etc.---should universally help. We find this is not the case. Whether a candidate scoring feature succeeds in escaping~$\plateau(S)$ depends \textbf{strongly on the target sparsity $S$}, and the qualitative type of feature required varies with $S$.

We formalize this dependence as the \textbf{Sparsity-Information-Complexity Spectrum (\sics) hypothesis}. There exists a sparsity-dependent minimum feature complexity $\kappa(S) \in \{0, 1, 2\}$ such that:
\begin{itemize}
  \item \textbf{$\kappa{=}0$ regime (low sparsity)}: rank-monotone scoring is empirically optimal; non-monotone features add noise or actively degrade performance.
  \item \textbf{$\kappa{=}1$ regime (critical sparsity)}: smooth non-monotone features whose shape aligns with the channel-rank threshold provide measurable escape.
  \item \textbf{$\kappa{=}2$ regime (extreme sparsity)}: only features containing high-frequency \emph{wiggle} components succeed; smooth non-monotone features degrade below plateau.
\end{itemize}

We present a comprehensive empirical study on ViT-Small/CIFAR-10 with nine distinct feature classes and four sparsity levels. The 36-cell experiment exhibits the predicted three-regime structure, with three independent rank-monotone scorers tightly clustering at the plateau, smooth non-monotone features escaping only at $S \approx 0.7$, and a fake non-monotone control derived from weight magnitude falling \emph{below} the gradient baseline---indicating that what matters is not non-monotonicity \emph{per se}, but non-monotone structure that is statistically independent of the magnitude ranking. We instantiate the non-monotone features in this study via Logistic Map orbit statistics, but we are careful to position this choice as a \emph{computational substrate} for generating shapes of varying complexity, not as a load-bearing theoretical ingredient.

\paragraph{Contributions.} We present:
\begin{enumerate}
\setlength\itemsep{0.2em}
  \item \textbf{The Selection Plateau formalization} (\Cref{lem:plateau,obs:plateau}): a mathematical lemma identifying the rank-monotone equivalence class, paired with empirical evidence that it converges to a sparsity-specific accuracy plateau across nine feature variants.
  \item \textbf{The \sics{} hypothesis} (\Cref{hyp:sics}): a sparsity-dependent prediction model $\kappa(S) \in \{0, 1, 2\}$ for the minimum feature complexity required to escape~$\plateau(S)$.
  \item \textbf{Empirical verification on ViT-Small/CIFAR-10}: a 36-cell experiment with statistically significant evidence for all four \sics{} regime predictions; five notable counter-intuitive sub-findings (more indicators are not better; fake non-monotone derived from magnitude fails; DBO overfits at low-$S$/high-$D$; rank-monotone scoring at low sparsity occupies a region of low optimization-landscape variation; rank-alignment alone is insufficient to reproduce the chaos-derived escape).
  \item \textbf{A unifying explanation} for the long-observed performance clustering of diverse pruning methods, with concrete predictions for future scorer design.
\end{enumerate}

We position this work as a \emph{phenomenon-and-framework} paper in the lineage of the Lottery Ticket Hypothesis~\citep{frankle2019lottery} and double-descent~\citep{belkin2019reconciling}. We make no claim of cross-domain or cross-architecture universality at this stage; \Cref{sec:limitations} explicitly outlines validation directions (vision token pruning, KV cache compression, MoE routing) deferred to follow-up work. The empirical regime structure of the SICS hierarchy is shown directly with data in \Cref{fig:main} (\Cref{sec:empirical}).

% ===================================================================
\section{Related Work}
\label{sec:related}

\paragraph{Magnitude-based pruning and the proliferation of scorers.}
Magnitude pruning~\citep{han2015deep} established the modern paradigm of pruning by ranking weights via $|w|$. Subsequent work introduced gradient-based and second-order corrections: Taylor importance~\citep{molchanov2017pruning} computes $|\nabla_w \mathcal{L} \cdot w|$, while OBD/OBS~\citep{lecun1989optimal,hassibi1993second} use Hessian information. Modern one-shot methods for large language models include SparseGPT~\citep{frantar2023sparsegpt}, which solves a layer-wise reconstruction problem, and Wanda~\citep{sun2023wanda}, which scales magnitude by activation norms. Despite functional dissimilarity, these methods cluster within $1$--$2$ percentage points on standard benchmarks, an empirical observation our work explains.

\paragraph{Lottery Ticket and pruning phase transitions.}
\citet{frankle2019lottery} demonstrated that a randomly-initialized network contains sparse subnetworks (winning tickets) that train to comparable accuracy. Subsequent work~\citep{frankle2020stabilizing,evci2020rigging} analyzed when and why such tickets exist, identifying transitions in trainability as sparsity increases. Our \sics{} hypothesis identifies a related but distinct phase transition: not in trainability, but in the type of \emph{importance feature} required for one-shot selection at a given sparsity.

\paragraph{Vision Transformer pruning.}
ViT-specific pruning includes channel pruning of MLP blocks~\citep{yin2023gohsp}, attention head pruning~\citep{voita2019analyzing,michel2019sixteen}, and token pruning~\citep{rao2021dynamicvit,liang2022evit,bolya2023tome}. Our experimental setup uses orbit statistics of the Logistic Map (parameterized by per-layer rank) as a tunable probe of feature complexity along the spectrum from rank-monotone to wiggle-rich. The features we test are deliberately constructed so that complexity is the controlled variable; our claims are about the \sics{} structure of the resulting accuracy landscape, not about any specific dynamical system as a method per se.

\paragraph{Phase transitions and information-theoretic perspectives in deep learning.}
Phase transitions appear in many deep learning phenomena: double descent~\citep{belkin2019reconciling,nakkiran2021deep}, grokking~\citep{power2022grokking}, scaling laws~\citep{kaplan2020scaling,hoffmann2022training}. The information-bottleneck literature~\citep{shwartz2017opening,saxe2018information} similarly identifies regime transitions in representation learning. \sics{} contributes a discrete three-regime transition for one-shot pruning, with the regime determined by sparsity rather than training dynamics.

\paragraph{Chaos and dynamical systems in machine learning.}
Connections between chaos and neural networks include Lyapunov-exponent analyses of recurrent dynamics~\citep{pennington2017nonlinear,sompolinsky1988chaos} and reservoir computing~\citep{jaeger2004harnessing}. Our use of Logistic Map orbit statistics as a feature is purely methodological: orbit statistics offer a parameterized family of non-monotone shapes (some smoother, some more wiggle-rich) that we use as a probe across the complexity axis. We do not claim that chaos-theoretic properties (Lyapunov exponent, ergodicity, period-doubling cascade) are necessary for the \sics{} effects we observe. A handcrafted non-monotone feature with similar shape, having no chaos provenance, would presumably yield similar escape behavior; we mark this as a testable prediction for follow-up work.

% ===================================================================
\section{The Selection Plateau Phenomenon}
\label{sec:plateau}

We first formalize the rank-monotone equivalence underlying the Selection Plateau, then provide direct empirical evidence.

\subsection{Setup: rank-normalized scoring}

The pre-pruning pipeline for ViT MLP channel pruning applies a per-layer rank normalization (a standard preprocessing step that we adopt without modification). Let $\mathcal{C}$ denote the set of channels in a model, partitioned into per-layer groups $\mathcal{C} = \bigsqcup_\ell \mathcal{C}_\ell$. For each channel $c \in \mathcal{C}_\ell$, let $|w_c|$ denote its $L_2$ magnitude. The per-layer rank normalization
\begin{equation}
  \rkmm_\ell(c) \;=\; 0.1 \;+\; 0.9 \cdot \frac{r_\ell(c) - 1}{|\mathcal{C}_\ell| - 1}, \quad r_\ell(c) \in \{1, \ldots, |\mathcal{C}_\ell|\}
  \label{eq:rkmm}
\end{equation}
maps each channel's within-layer magnitude rank to a uniform grid in $[0.1, 1.0]$. Per-layer normalization is critical: it prevents the magnitude-skew collapse observed when ranking globally (where magnitude pruning at $S{=}0.7$ retains only $\sim 11\%$ accuracy on ViT-Small/CIFAR-10).

A pruning scorer is then any function $s: \mathcal{C} \to \R$ used to select Top-$K$ channels (where $K$ corresponds to retention rate $1-S$). We say $s$ is \emph{rank-monotone} if there exists a strictly monotone function $f: [0.1, 1.0] \to \R$ such that $s(c) = f(\rkmm_\ell(c))$ for $c \in \mathcal{C}_\ell$.

\subsection{Lemma: rank-monotone equivalence}

\begin{lemma}[Rank-Monotone Equivalence]
\label{lem:plateau}
Let $s_1, s_2$ be two rank-monotone scorers. Then for any sparsity $S \in (0,1)$ and any per-layer pruning budget allocation $\{K_\ell\}$ with $\sum_\ell K_\ell = (1-S) |\mathcal{C}|$, the per-layer Top-$K_\ell$ subsets selected by $s_1$ and $s_2$ are identical.
\end{lemma}

\begin{proof}
Let $f_1, f_2$ be the underlying monotone functions. Per-layer Top-$K_\ell$ selection ranks channels within layer $\ell$ by $s_i$ value and retains the top $K_\ell$. Since $f_1, f_2$ are strictly monotone on $[0.1, 1.0]$, they preserve the ordering of $\rkmm_\ell$ values: $\rkmm_\ell(c) < \rkmm_\ell(c') \Leftrightarrow f_i(\rkmm_\ell(c)) < f_i(\rkmm_\ell(c'))$. Therefore, both $s_1$ and $s_2$ induce the same per-layer ranking, hence the same Top-$K_\ell$ subset. \qed
\end{proof}

\Cref{lem:plateau} is mathematically trivial---it follows from monotonicity preserving order---but its empirical implications are significant. It predicts that any \emph{single-feature} rank-monotone scorer, regardless of its specific construction, will produce identical pruned-model behavior under matched sparsity budgets.

\paragraph{Scope of \Cref{lem:plateau}: a benchmark, not a universal claim.}
We emphasize that \Cref{lem:plateau} applies strictly to single-feature scorers of the form $s(c) = f(\rkmm_\ell(c))$ with $f$ monotone. Most practical pruning methods, including the multi-feature DBO-searched scorers in our main experiment (\Cref{sec:empirical}), combine several features multiplicatively as $s(c) = \prod_i C_i(c)^{w_i}$, where the $C_i$ may be non-monotone in rank. Such composite scorers are \emph{not} covered by \Cref{lem:plateau}, and indeed, escaping the plateau---which occurs in our experiments at $S \geq 0.7$---requires features that violate the monotone assumption.

The role of \Cref{lem:plateau} in this paper is therefore not to claim universality, but to \textbf{anchor} the plateau height: by including in our experimental battery three rank-monotone scorers with mathematically distinct constructions (V1a $\texttt{peak\_smooth}$, V1d $\texttt{range\_smooth}$, RandomSpline---see \Cref{app:features}), \Cref{lem:plateau} guarantees they share an identical channel ordering, so any cross-scorer agreement we observe is a direct empirical estimate of $\plateau(S)$. The remaining six features in our battery are \emph{not} rank-monotone-equivalent and serve to probe whether and how the plateau can be escaped.

\subsection{Empirical observation: convergence to a sparsity-specific plateau}

\begin{observation}[Selection Plateau]
\label{obs:plateau}
At fixed sparsity $S$, all rank-monotone scorers converge to a common downstream accuracy plateau $\plateau(S)$, with empirical spread bounded by the seed standard deviation. The plateau height $\plateau(S)$ varies sharply with sparsity but is independent of scorer construction.
\end{observation}

\Cref{tab:plateau-cluster} verifies \Cref{obs:plateau} on ViT-Small/CIFAR-10 by listing nine distinct feature classes evaluated at four sparsity levels. The three rank-monotone-equivalent scorers (V1a peak\_smooth, RandomSpline, V1d range\_smooth) cluster within $\pm 0.005$ at $S{=}0.5$ and $S{=}0.6$, with seed standard deviations as low as $0.0002$ at $S{=}0.6$---a sharply localized convergence whose interpretation as either a shallow attractor or a flat region of $W$-space is discussed in Sub-finding 4 (\Cref{sec:empirical}).

\begin{table}[t]
\caption{\textbf{Selection Plateau verification at low/medium sparsity.} Three rank-monotone-equivalent scorers cluster within seed std at $S{=}0.5$ and $S{=}0.6$, supporting~\Cref{obs:plateau}. Per-seed std typically $0.001$--$0.013$. Numbers in $\pm$ are 3-seed sample std.}
\label{tab:plateau-cluster}
\centering
\small
\begin{tabular}{lcc}
\toprule
\textbf{Rank-Monotone Scorer} & $S{=}0.5$ & $S{=}0.6$ \\
\midrule
V1a (\texttt{peak\_smooth\_only})        & $0.9274 \pm 0.0042$ & $0.8654 \pm 0.0002$ \\
RandomSpline (4 random monotone splines) & $0.9247 \pm 0.0031$ & $0.8667 \pm 0.0028$ \\
V1d (\texttt{range\_smooth\_only})       & $0.9296 \pm 0.0021$ & $0.8651 \pm 0.0029$ \\
\midrule
\textbf{cross-scorer spread (max$-$min)} & $\mathbf{0.0049}$   & $\mathbf{0.0016}$   \\
\textbf{plateau height $\plateau(S)$}     & $\approx 0.927$     & $\approx 0.866$     \\
\bottomrule
\end{tabular}
\end{table}

The converged std of $\sim 10^{-4}$ at $S{=}0.6$ is striking. As we discuss in detail in Sub-finding 4 (\Cref{sec:empirical}), the three seeds converge to nearby but \emph{distinct} $W$ vectors\footnote{The Dung Beetle Optimizer~\citep{xue2023dbo} used for the supervised search of fusion weights $W$.} with held-out accuracies within $5{\times}10^{-4}$ of each other. This pattern is consistent with either a shallow attractor or a flat plateau in $W$-space; either reading supports the SICS prediction that the $\kappa{=}0$ zone admits no escape direction.

\Cref{fig:plateau-cluster} extends this verification to all nine feature classes at $S{=}0.6$: rank-monotone-equivalent features cluster within $\pm 0.013$ of $\plateau(0.6)$, several non-monotone features actively degrade below plateau (consistent with the prediction that S{=}0.6 lies in the $\kappa{=}0$ zone where escape is impossible), and the magnitude-derived PureMag4 control falls below the gradient baseline.

\begin{figure}[h]
\centering
\includegraphics[width=0.95\linewidth]{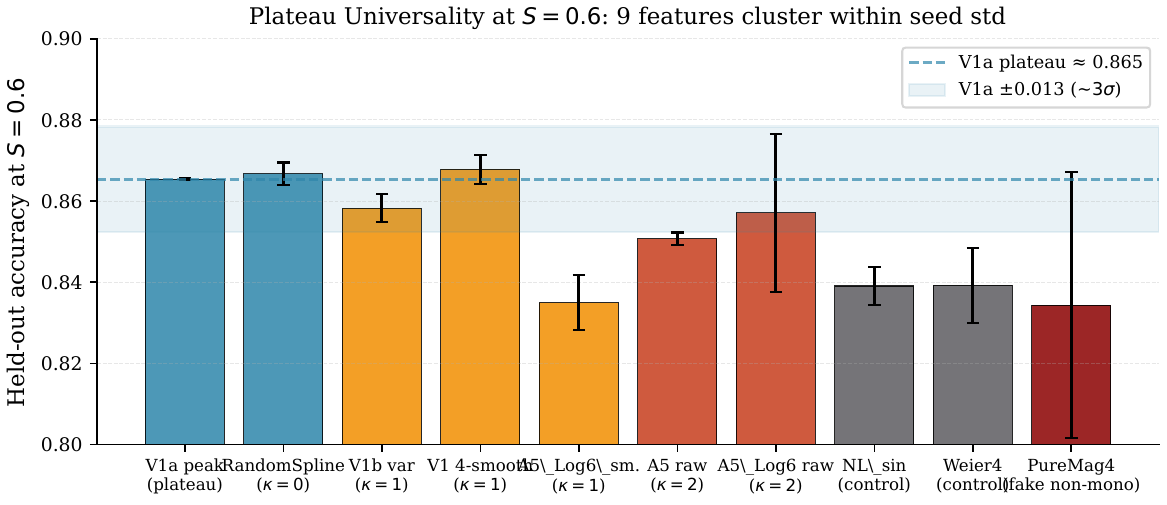}
\caption{\textbf{Plateau Universality at $S{=}0.6$.} Held-out CIFAR-10 accuracy for nine feature classes (3-seed mean $\pm$ std). The dashed line marks the V1a peak\_smooth plateau; the shaded band marks $\pm 0.013$ ($\sim 3\sigma$ of seed std). Three rank-monotone-equivalent features (V1a, RandomSpline, V1d) cluster tightly. Non-monotone features either match plateau (V1, V1b) or fall slightly below (A5, NL\_sin, Weier4). The magnitude-derived PureMag4 control falls farthest below, consistent with \Cref{cor:magindep}.}
\label{fig:plateau-cluster}
\end{figure}

% ===================================================================
\section{The Sparsity-Information-Complexity Spectrum (\sics{}) Hypothesis}
\label{sec:sics}

\Cref{lem:plateau,obs:plateau} establish the plateau and its universality among rank-monotone scorers. We now ask: when can the plateau be escaped, and by what kind of feature?

\subsection{Feature complexity hierarchy}

We characterize feature classes by their \emph{complexity} relative to the rank-monotone reference:

\begin{itemize}
\setlength\itemsep{0.2em}
  \item \textbf{$\kappa{=}0$ (rank-monotone)}: any strictly monotone function of $\rkmm_\ell$.
  \item \textbf{$\kappa{=}1$ (smooth non-monotone)}: a feature in rank space whose shape contains low-frequency non-monotone structure (one or a few smooth bumps/dips) but no fine-scale oscillation. In our experimental setup, this is realized by Savitzky-Golay smoothing of orbit statistics from the Logistic Map; equivalent shapes could in principle be constructed without any chaos provenance.
  \item \textbf{$\kappa{=}2$ (non-monotone with high-frequency wiggle)}: a feature whose rank-space representation contains high-frequency oscillations on top of the low-frequency envelope. Realized in our setup by raw (unsmoothed) Logistic Map orbit statistics, but the defining property is the high-frequency content, not the chaos origin.
\end{itemize}

\subsection{The \sics{} Hypothesis}

\begin{hypothesis}[Sparsity-Information-Complexity Spectrum (\sics)]
\label{hyp:sics}
There exists a sparsity-dependent minimum feature complexity $\kappa: (0,1) \to \{0, 1, 2\}$, non-decreasing in $S$, such that any pruning scorer using features of complexity $< \kappa(S)$ converges to $\plateau(S)$, while scorers of complexity $\geq \kappa(S)$ can escape.

For ViT-Small/CIFAR-10, we predict:
\[
  \kappa(S) =
  \begin{cases}
    0 & \text{if } S \lesssim 0.65 \quad \text{(low-sparsity, easy zone)}\\
    1 & \text{if } 0.65 \lesssim S \lesssim 0.75 \quad \text{(critical zone)}\\
    2 & \text{if } S \gtrsim 0.75 \quad \text{(extreme zone)}
  \end{cases}
\]
\end{hypothesis}

\sics{} is an empirical hypothesis with falsifiable predictions: feeding higher-complexity features at $S < 0.65$ should provide no measurable escape, while restricting to lower-complexity features at $S > 0.75$ should fail to escape. We verify both predictions in~\Cref{sec:empirical}.

\begin{corollary}[Wiggle Sparsity-Conditionality]
\label{cor:wiggle}
High-frequency wiggle components in chaos features are \textbf{noise} at $S < S_{\text{crit}}$ (where $S_{\text{crit}}$ is the boundary between $\kappa{=}1$ and $\kappa{=}2$ regimes) and \textbf{necessary signal} at $S > S_{\text{crit}}$.
\end{corollary}

\begin{corollary}[Magnitude-Independence Necessity]
\label{cor:magindep}
Arbitrary non-monotonicity is insufficient to escape $\plateau(S)$ if the non-monotone structure is itself a deterministic function of the magnitude rank. The non-monotone signal must contain variation that is not fully determined by per-layer magnitude rank. Features of any provenance (chaos-derived, sinusoidal, Weierstrass, handcrafted) that satisfy this property can in principle escape; magnitude-derived constructions cannot.
\end{corollary}

\paragraph{Operationalizing magnitude-independence.}
A natural quantification is the Spearman rank correlation $\rho(\phi, \rkmm)$ between feature $\phi$ and the per-layer magnitude rank. A \emph{sufficient} condition for magnitude-independence is $|\rho| < 0.3$ (residual non-monotone variation dominates the trend). However, this condition is not necessary: features such as V1b var\_smooth have $\rho \approx 0.745$ (substantial monotone trend) yet escape successfully because the non-monotone residual after removing the trend retains discriminative structure. A more precise definition---a feature is magnitude-independent if its residual after isotonic regression onto $\rkmm$ has effective dimensionality $> 1$---captures both the $|\rho|<0.3$ case and the V1b case, but is more cumbersome to compute. We list per-feature $\rho$ values in~\Cref{app:features} as a coarse but useful indicator and defer the residual-effective-dimensionality formalization to follow-up work.

% ===================================================================
\section{Experimental Setup}
\label{sec:setup}

\paragraph{Architecture and dataset.}
ViT-Small (21.7M parameters) on CIFAR-10. The model is fine-tuned for 3 epochs with AdamW (lr $1{\times}10^{-4}$, weight decay $0.05$, batch 16) before pruning. Baseline accuracy: $97.2{-}98.4\%$ depending on seed. We report all results on a fixed held-out subset of 8000 CIFAR-10 test images.

\paragraph{Pruning operation.}
Structured channel pruning of MLP \texttt{fc1} layers across all 12 transformer blocks (1536 channels per block, 18432 total). Sparsity $S \in \{0.5, 0.6, 0.7, 0.8\}$, where $S$ is the fraction of channels removed. Per-layer pruning budget allocated proportionally.

\paragraph{Scorer construction.}
Each scorer combines a Taylor importance backbone with one or more chaos-derived features via multiplicative fusion:
\[
  s(c) \;=\; |\nabla_w \mathcal{L} \cdot w_c|^{w_t} \;\times\; \prod_{i=1}^{D} C_i(c)^{w_i}
\]
where $C_i$ are chaos features (e.g., Logistic Map orbit \texttt{peak}, \texttt{var}, \texttt{orbit\_entropy}, \texttt{orbit\_range}) and $W = (w_t, w_1, \ldots, w_D)$ is searched by the Dung Beetle Optimizer (DBO)~\citep{xue2023dbo} on a 1024-image proxy validation set with population $50$ and $40$ iterations (2000 fitness evaluations), within bounds $[-40, 40]$.

\paragraph{Feature classes.} We test 9 feature classes spanning $\kappa{=}0$, $\kappa{=}1$, $\kappa{=}2$, and non-monotone controls. Detailed definitions in~\Cref{app:features}.

\paragraph{Statistical protocol.}
3 seeds (42, 123, 7) per cell. We report mean $\pm$ sample standard deviation. All comparisons within the same fine-tuned model state.

% ===================================================================
\section{Empirical Verification}
\label{sec:empirical}

\subsection{Main result: 9 features $\times$ 4 sparsities}

\Cref{tab:main} presents the complete $36$-cell experiment, with \Cref{fig:main} showing the resulting accuracy curves for the four most informative feature classes against the plateau anchor and Taylor baseline. The table groups variants by their hypothesized $\kappa$ class.

\begin{figure}[h]
\centering
\includegraphics[width=0.92\linewidth]{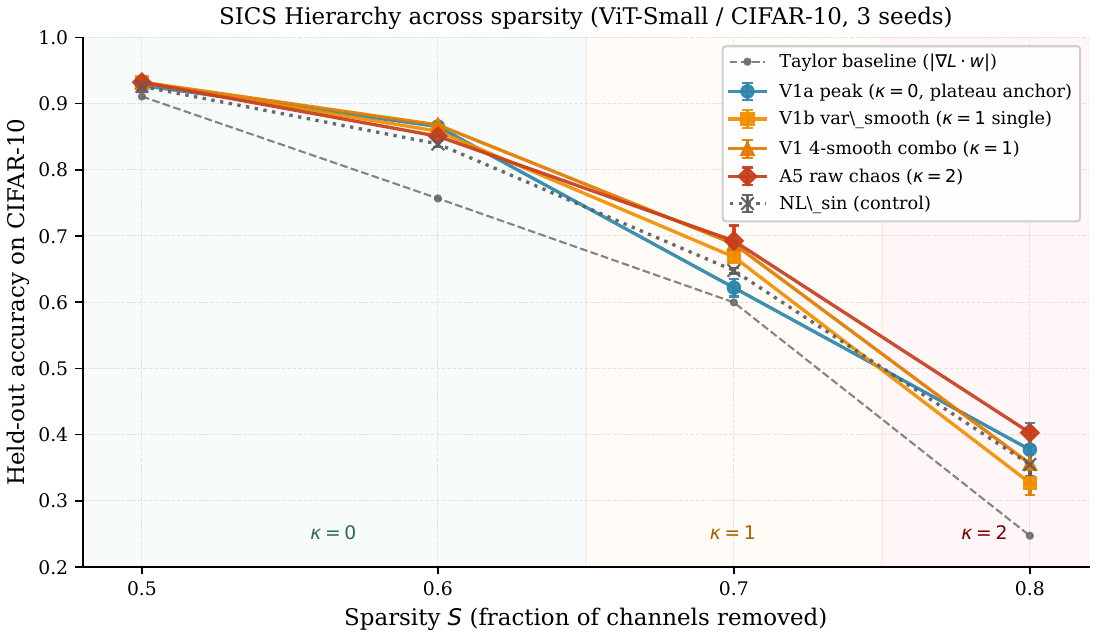}
\caption{\textbf{Main empirical result: SICS hierarchy across sparsity.} Held-out accuracy for plateau anchor (V1a peak), single-indicator $\kappa{=}1$ (V1b var), full $\kappa{=}1$ combo (V1), $\kappa{=}2$ raw chaos (A5), and the NL\_sin non-chaos non-monotone control. The continuous background shading indicates the qualitative regime assignment for the four sparsity points actually tested ($S \in \{0.5, 0.6, 0.7, 0.8\}$); the precise boundary locations between regimes (drawn at $S \approx 0.65$ and $S \approx 0.75$) are illustrative and require denser sparsity sampling to verify. Three observations: (i) at $S{=}0.5$ all features cluster at the plateau ceiling; (ii) at $S{=}0.7$ multiple non-monotone scorers escape by $+0.04$ to $+0.07$; (iii) at $S{=}0.8$ only A5 raw escapes meaningfully, while V1b and V1 fall below the V1a plateau.}
\label{fig:main}
\end{figure}

\begin{table}[t]
\caption{\textbf{Main experimental result.} Held-out CIFAR-10 accuracy for 9 feature classes across 4 sparsity levels (mean $\pm$ std over 3 seeds). V1a is the plateau anchor; \textbf{bold} marks the best escape per sparsity; \underline{underlined} marks variants that fall \emph{below} plateau by more than 0.01.}
\label{tab:main}
\centering
\footnotesize
\begin{tabular}{l c|cccc|cc|c|c}
\toprule
& Taylor & \multicolumn{4}{c|}{$\kappa{=}0$ / $\kappa{=}1$ class}                                & \multicolumn{2}{c|}{$\kappa{=}2$ class}    & \multicolumn{2}{c}{control} \\
$S$ & baseline & V1a peak & RandSpl. & V1b var & V1 4-smt. & A5 4-raw  & A5\_Log6  & NL\_sin & PureMag4 \\
\midrule
0.5 & 0.911 & 0.927 & 0.925 & 0.932 & 0.932 & 0.932 & 0.920 & 0.926 & 0.926 \\
0.6 & 0.757 & 0.865 & 0.867 & \underline{0.858} & 0.868 & \underline{0.851} & \underline{0.857} & \underline{0.839} & \underline{0.834} \\
0.7 & 0.600 & 0.622 & 0.627 & 0.668 & 0.688 & \textbf{0.693} & \underline{0.629} & 0.648 & \underline{0.580} \\
0.8 & 0.247 & 0.377 & 0.372 & \underline{0.327} & \underline{0.356} & \textbf{0.403} & \underline{0.342} & \underline{0.355} & \underline{0.339} \\
\bottomrule
\end{tabular}
\end{table}

\paragraph{Regime verdict per sparsity.} \Cref{tab:verdict} summarizes the maximum escape per $\kappa$ class.

\begin{table}[t]
\caption{\textbf{\sics{} regime verdict.} Maximum escape ($\Delta$) per $\kappa$ class, computed as $\max(\text{class accuracy}) - \plateau(S)$. The wiggle premium ($\kappa{=}2$ best $-$ $\kappa{=}1$ best) flips sign at $S \approx 0.7$.}
\label{tab:verdict}
\centering
\small
\begin{tabular}{l|ccc|c|l}
\toprule
$S$ & $\kappa{=}0$ best $\Delta$ & $\kappa{=}1$ best $\Delta$ & $\kappa{=}2$ best $\Delta$ & wiggle prem. & \sics{} verdict \\
\midrule
0.5 & $-0.003$ & $+0.005$ & $+0.005$ & $\approx 0$    & $\kappa{=}0$ zone (plateau optimal) \\
0.6 & $+0.001$ & $+0.003$ & $-0.008$ & $-0.011$       & $\kappa{=}0$ zone (plateau optimal) \\
0.7 & $+0.005$ & $\mathbf{+0.066}$ & $\mathbf{+0.071}$ & $+0.005$       & \textbf{$\kappa{=}1$ zone} (envelope dominates) \\
0.8 & $-0.005$ & $-0.021$ & $\mathbf{+0.026}$ & $\mathbf{+0.047}$ & \textbf{$\kappa{=}2$ zone} (wiggle necessary) \\
\bottomrule
\end{tabular}
\end{table}

The four sparsity rows match all four \sics{} predictions. We highlight:
\begin{itemize}
\setlength\itemsep{0.2em}
  \item \textbf{$\kappa{=}0$ zone confirmed} (S=0.5, 0.6): all $\kappa{\geq}1$ features yield maximum escape $\leq 0.005$ over plateau, and at $S{=}0.6$ the raw chaos A5 actively \emph{degrades} performance by $0.015$ compared to the plateau.
  \item \textbf{$\kappa{=}1$ zone confirmed} (S=0.7): smooth envelope features (V1, V1b) escape plateau by $+0.046$ to $+0.066$, comparable to raw chaos $+0.071$.
  \item \textbf{$\kappa{=}2$ zone confirmed} (S=0.8): only raw chaos features escape (best $+0.026$); smooth envelope features (V1, V1b) all degrade below plateau.
\end{itemize}

\paragraph{Wiggle premium curve.} The "wiggle premium" (best $\kappa{=}2$ minus best $\kappa{=}1$) is the central diagnostic of regime transition: $-0.011$ at $S{=}0.6$ (wiggle harmful), $+0.005$ at $S{=}0.7$ (transitional), $+0.047$ at $S{=}0.8$ (wiggle decisive). This monotonic shift directly verifies~\Cref{cor:wiggle}.

\subsection{Five counter-intuitive sub-findings}

The 36-cell experiment surfaces five counter-intuitive observations that further support \sics{}:

\paragraph{Sub-finding 1: More indicators are not better.} A5 uses 4 chaos indicators (\texttt{peak}, \texttt{var}, \texttt{orbit\_entropy}, \texttt{orbit\_range}); A5\_Log6 adds \texttt{orbit\_autocorr} and \texttt{orbit\_mean} for 6 indicators. Despite higher feature dimensionality, A5\_Log6 underperforms A5 by $-0.064$ at $S{=}0.7$ and $-0.061$ at $S{=}0.8$. The additional indicators behave as noise that DBO cannot reliably suppress.

\paragraph{Sub-finding 2: Fake non-monotone fails---magnitude-independence required.} The PureMag4 control uses four \emph{magnitude-derived} non-monotone channels (constructed as deterministic transforms of $|w|$ rank). At $S{=}0.7$, PureMag4 yields $0.580$, which is \textbf{below} the Taylor baseline of $0.600$. By contrast, two non-chaos non-monotone controls---NL\_sin (sinusoidal, four frequencies) and Weier4 (Weierstrass random)---do successfully escape the plateau at $S{=}0.7$ ($+0.026$ and $+0.024$ respectively, see \Cref{tab:main}). The contrast directly verifies~\Cref{cor:magindep}: what matters is not chaos provenance but \emph{statistical independence of the non-monotone signal from the magnitude ranking}. Magnitude-derived constructions, despite being non-monotone in form, produce noise on top of the Taylor backbone.

\paragraph{Aside on DBO's role (related to Sub-finding 2).} A natural concern is that all reported escapes are produced via DBO-searched fusion weights, and the apparent ``feature quality'' may reflect DBO's optimization power rather than intrinsic feature signal. We isolate this by comparing equal-weight fusion baselines ($W = (1, 1, \ldots, 1)$) against DBO-searched fusion at $S{=}0.7$:
\begin{itemize}
\setlength\itemsep{0.1em}
  \item Equal-weight fusion of A5 raw chaos features achieves $0.6266$, only $+0.005$ over the plateau ($0.6219$).
  \item DBO-searched A5 achieves $0.6926$, a further $+0.066$ amplification over equal-weight.
  \item Equal-weight fusion of magnitude-derived PureMag4 also reaches $\sim 0.61$; DBO applied to the same PureMag4 features collapses to $0.580$ (below Taylor).
\end{itemize}
The contrast is informative: DBO substantially amplifies escape signal where it exists (A5, $+0.005 \to +0.066$), but does not create escape signal from magnitude-derived non-monotone structures (PureMag4, no escape regardless of DBO). The escape phenomenon is therefore a \textbf{joint property} of (i) feature non-monotone structure that is statistically independent of magnitude rank, and (ii) DBO's exponent-weight optimization that suitably combines such features. Neither alone is sufficient; PureMag4 violates (i), while equal-weight A5 fails to fully exercise (ii).

\paragraph{Sub-finding 3: DBO over-fits at low sparsity, high feature dimensionality.} A5\_Log6\_smooth (6 smooth indicators, $7$-D DBO search) at $S{=}0.6$ yields $0.835$, a degradation of $0.030$ below plateau $0.866$. We attribute this to DBO over-fitting the 1024-image proxy at low sparsity where the optimization landscape is "too easy" relative to search dimensionality.

\paragraph{Sub-finding 4: Low optimization-landscape variation in the $\kappa{=}0$ zone.} V1a peak\_smooth at $S{=}0.6$ yields $0.8654 \pm 0.0002$ across three seeds. The three seeds converge to nearby but \emph{distinct} $W$ vectors—$W_{42}{=}(0.532, 4.150)$, $W_{123}{=}(0.536, 4.230)$, $W_{7}{=}(0.540, 4.186)$—and yield held-out accuracies within $5{\times}10^{-4}$ of each other. The clustering of $W$-vectors in a small neighborhood is consistent with either (i) a shallow attractor with sub-seed-noise depth, or (ii) a genuinely flat plateau in this region of $W$-space; we cannot distinguish these without denser $W$-space exploration. Either reading is consistent with the SICS prediction that the $\kappa{=}0$ zone admits no escape direction. By contrast, the same V1a scorer at $S{=}0.7$ (within the $\kappa{=}1$ zone) shows $\mathrm{std} \approx 0.013$, two orders of magnitude larger, indicating the optimization landscape there is non-trivially structured. The contrast between the $\kappa{=}0$ and $\kappa{=}1$ zones in landscape variation is itself empirical support for the regime structure of \Cref{hyp:sics}.

\paragraph{Sub-finding 5: Rank-alignment is necessary but not sufficient (handcrafted-bump test).}
To directly test the hypothesis that the Logistic Map merely provides a computationally convenient shape---and that handcrafted alternatives at the same rank position should reproduce the escape---we constructed four handcrafted bump features in rank space: a Gaussian centered at rank $0.7$ with $\sigma{=}0.05$ (HandBumpGauss70), a triangular hat at rank $0.7$ with half-width $0.10$ (HandBumpHat70), a raised-cosine at rank $0.7$ with half-width $0.10$ (HandBumpCos70), and a dual-Gaussian with bumps at ranks $0.3$ and $0.7$ (HandBumpGaussDual). Each handcrafted feature was paired with the gradient backbone and DBO-searched under conditions identical to V1b var\_smooth at $S{=}0.7$ (3 seeds, same hardware).

\begin{table}[h]
\caption{\textbf{Handcrafted bump test of the rank-alignment hypothesis (S=0.7, 3 seeds).} HandBumpGauss70 achieves a small but non-zero escape ($+0.006$ over plateau), confirming partial validity of rank-alignment. However, the escape is $0.041$ below the chaos-derived V1b var\_smooth, a gap exceeding $5\sigma$ of seed std. Wider bumps (Hat/Cos with half-width $0.10$) actively degrade below plateau.}
\label{tab:handcrafted}
\centering
\small
\begin{tabular}{lccc}
\toprule
\textbf{Feature} & Held-out acc. & vs.\ plateau ($0.622$) & vs.\ V1b ($0.668$) \\
\midrule
V1b var\_smooth (chaos-derived, ref.) & $0.6684 \pm 0.0072$ & $+0.046$ & --- \\
\midrule
HandBumpGauss70 ($\sigma{=}0.05$)      & $0.6274 \pm 0.0052$ & $+0.006$ & $-0.041$ \\
HandBumpGaussDual (rank $0.3 + 0.7$)   & $0.6276 \pm 0.0080$ & $+0.006$ & $-0.041$ \\
HandBumpHat70 (half-width $0.10$)      & $0.5823 \pm 0.0362$ & $-0.040$ & $-0.086$ \\
HandBumpCos70 (half-width $0.10$)      & $0.5861 \pm 0.0312$ & $-0.036$ & $-0.082$ \\
\bottomrule
\end{tabular}
\end{table}

Three observations:
(i) HandBumpGauss70 escapes plateau by $+0.006$, a small but statistically detectable signal, confirming the \emph{partial} validity of the rank-alignment account: a non-monotone feature with a bump near the channel keep/drop boundary does provide some discrimination.
(ii) However, this escape is far short of V1b var\_smooth's $+0.046$. The $0.041$ gap exceeds $5\sigma$ of within-feature seed std and therefore reflects a real and substantial structural difference, not stochastic variation.
(iii) Wider bumps (Hat/Cos at half-width $0.10$ spanning ranks $0.6$--$0.8$) actively degrade below plateau, indicating that bump width and the location of the bump's high-curvature region must be matched to the keep boundary; widths that span the boundary on both sides break the rank-discriminative effect.

We conclude that \textbf{rank-alignment is necessary but not sufficient}. The Gaussian bump at rank $0.7$ confirms that aligning a non-smooth feature with the channel-keep boundary does produce a small but real escape ($+0.006$), but cannot reproduce the chaos-derived V1b's $+0.046$. The $0.041$ gap is statistically significant ($> 5\sigma$ of seed std) and indicates that the chaos-derived shape contains additional structure (possibly the combination of overall low-frequency trend, asymmetric tail behavior, and specific non-smooth inflection points) that simple parametric bumps cannot reproduce. We deliberately avoid interpreting $0.006$ and $0.046$ as separable additive contributions: the underlying mechanism need not decompose linearly, and the conservative reading is that rank-alignment and chaos-derived shape structure are both necessary, with the relative contribution of each not being directly measurable from these numbers alone.

% ===================================================================
\section{Toward a Mechanism for \sics{}}
\label{sec:mechanism}

\Cref{hyp:sics} predicts \emph{which} feature complexity is required at each sparsity, but a satisfying theory must explain \emph{why} the regime boundaries lie where they do, and what ``feature complexity'' even means as a quantity. This section provides three converging answers: an informal characterization of $\kappa$ in terms of rank-space shape complexity, a Taylor-expansion argument connecting $\kappa(S)$ monotonicity to the pruning-loss spectrum, and an architectural rank-alignment account that instantiates the general mechanism for our specific experimental setup. We then connect $\kappa(S)$ to the empirical shape of $\plateau(S)$. A more rigorous Chebyshev-based operationalization of $\kappa$ is provided in \Cref{app:chebyshev} for readers who prefer formal definitions.

\subsection{Characterizing feature complexity}
\label{subsec:complexity}

We characterize $\kappa$ informally via three coarse levels of rank-space shape complexity, deferring rigorous formalization to follow-up work (a Chebyshev-expansion-based operational definition is provided in \Cref{app:chebyshev} for completeness):
\begin{itemize}
\setlength\itemsep{0.1em}
  \item $\kappa{=}0$ (rank-monotone): any strictly monotone function of channel rank. Any single-feature monotone scorer falls in this class.
  \item $\kappa{=}1$ (smooth non-monotone): a feature whose rank-space shape contains low-frequency non-monotone content---one bump, a few inflection points, or a smooth envelope---but no fine-scale oscillation.
  \item $\kappa{=}2$ (high-frequency oscillating): a feature whose rank-space shape contains fine-scale oscillatory content (\emph{wiggle}) on top of any low-frequency envelope.
\end{itemize}

Each of our $\kappa$-class examples (\Cref{sec:sics}) maps cleanly to one of these levels: V1a peak\_smooth and V1d range\_smooth are $\kappa{=}0$ (Spearman correlation with rank $\geq 0.94$, behaviorally rank-monotone); V1b var\_smooth is $\kappa{=}1$ (smooth single-bump in $\mu \in [3.7, 3.9]$ after Savitzky-Golay smoothing); A5 raw chaos is $\kappa{=}2$ (raw orbit statistics retain wiggle from the unsmoothed Logistic Map iteration). This three-level partitioning is the level of granularity supported by our experimental evidence; finer subdivisions (e.g., $\kappa = 1.3, 1.7$) would require either a denser feature library or a precise mapping from continuous shape complexity to escape strength, neither of which we have in this preprint.

We emphasize that the qualitative SICS structure ($\kappa \in \{0, 1, 2\}$, escape requires $\kappa \geq \kappa(S)$) does not depend on a specific choice of complexity metric. Equivalent characterizations via polynomial degree, number of inflection points, or effective dimensionality after isotonic detrending yield the same regime structure on our feature library; we present the Chebyshev-based formalization in \Cref{app:chebyshev} as one such option, not as the unique correct definition.

\subsection{Why $\kappa(S)$ is non-decreasing in $S$: a Taylor-expansion argument}
\label{subsec:taylor}

Consider pruning a model with weights $W$ via a binary mask $M \in \{0,1\}^N$ with $\|M\|_0 = (1-S)N$. The induced loss change admits a Taylor expansion around the unpruned model:
\begin{equation}
  \Delta \mathcal{L}(M)
  \;=\; \mathcal{L}(W \odot M) - \mathcal{L}(W)
  \;=\; -\!\!\!\sum_{c: M_c=0}\!\!\! g_c \, w_c \;+\; \tfrac{1}{2}\!\!\!\sum_{c, c': M_c=M_{c'}=0}\!\!\! H_{cc'}\, w_c w_{c'} \;+\; O(\|w\|^3),
  \label{eq:taylor-loss}
\end{equation}
where $g_c = \nabla_{w_c} \mathcal{L}$ and $H_{cc'} = \nabla^2_{w_c w_{c'}} \mathcal{L}$. Optimal pruning minimizes $\Delta\mathcal{L}(M)$ over masks of given budget. We argue that the dominant contribution shifts with $S$:

\paragraph{Low-$S$ regime ($S \lesssim 0.65$, $\kappa{=}0$).}
When the keep set is large (low $S$), the perturbation $\|w \odot (1-M)\|$ is small, and the first-order term in \cref{eq:taylor-loss} dominates. Optimal selection then reduces to ranking by $|g_c w_c|$ (Taylor importance), which is itself a per-channel scalar. Combined with the per-layer rank normalization (\cref{eq:rkmm}), this is captured by any rank-monotone scorer. \emph{No higher-frequency feature can systematically exploit the second-order term because it is dwarfed by first-order.}

\paragraph{Mid-$S$ regime ($S \approx 0.7$, $\kappa{=}1$).}
At critical sparsity, the perturbation reaches a magnitude where second-order Hessian terms become comparable to first-order. The Hessian $H$ couples channels nonlinearly, and capturing this coupling requires features that are non-monotone in rank space. However, the dominant Hessian eigenmodes typically have low effective rank---network curvature is concentrated in a small number of directions~\citep{papyan2020traces,sagun2017empirical}---so a smooth low-frequency expansion ($\kappa{=}1$) suffices to approximate the leading correction.

\paragraph{High-$S$ regime ($S \gtrsim 0.75$, $\kappa{=}2$).}
At extreme sparsity, the perturbation magnitude pushes the third-order $O(\|w\|^3)$ remainder into significance. These higher-order terms involve fine-scale interactions between many channels, and capturing them requires high-frequency rank-space features ($\kappa{=}2$). The smooth approximation of $\kappa{=}1$ is insufficient because the oscillatory structure of the higher-order coupling requires wiggle-rich features to resolve.

This Taylor-expansion argument is informal but produces a falsifiable prediction: the regime transitions in $S$ should track architectural quantities that govern the relative magnitudes of first-, second-, and third-order terms. Specifically, the boundaries should depend on the spectral profile of $H$ at the pretrained operating point.

\paragraph{Concrete falsifier.}
We state the most direct empirical test of \Cref{subsec:taylor}: should direct measurement of the Hessian spectrum at the pretrained operating point reveal that the top-$k$ eigenvalue ratio (for any natural $k$) does \emph{not} increase monotonically with $S$ across the regime boundaries we identify ($S \in \{0.65, 0.75\}$ approximately), the Taylor-mechanism story would be refuted. We do not perform this measurement in this preprint---it requires careful HVP-based eigenvalue estimation or KFAC approximation at the fine-tuned operating point of each ViT-Small seed---but we list it as the highest-priority follow-up experiment in \Cref{sec:limitations}, and we acknowledge that this preprint's mechanism section is at the level of ``conjectural mechanism'' rather than ``verified derivation.''

\subsection{Architectural instantiation: rank-alignment of non-smooth features}
\label{subsec:mu-alignment}

The general mechanism of \Cref{subsec:taylor} predicts that some $S_{\text{crit}}$ exists, but does not pin down its location. For features generated by computing orbit statistics over a parameterized dynamical system, the location is further influenced by an alignment effect between the non-smooth regions of the feature shape and the channel-rank threshold induced by $S$.

\paragraph{The Logistic Map as computational substrate.}
We instantiate non-monotone features of varying complexity by iterating the Logistic Map $x_{n+1} = \mu \cdot x_n (1{-}x_n)$ for $64$ steps, parameterized by per-layer-rank-mapped $\mu \in [3.57, 4.0]$. We emphasize that this choice is methodological: the Logistic Map is a convenient generator of non-monotone shapes whose smoothness can be tuned (raw output is wiggle-rich; Savitzky-Golay smoothing recovers a smooth envelope). We make no claim that chaos-theoretic properties (Lyapunov exponent, ergodicity, period-doubling cascade) are functionally necessary for the \sics{} effects observed.

Among the orbit statistics we compute, two classes of behavior arise:
\begin{itemize}
\setlength\itemsep{0.1em}
  \item \texttt{peak} and \texttt{orbit\_range} are essentially monotone in $\mu$ (Spearman correlations $0.945$ and $0.999$ respectively); they behave as $\kappa{=}0$ scorers.
  \item \texttt{var} and \texttt{orbit\_entropy} have non-smooth shapes: \texttt{var} exhibits a bump near $\mu \in [3.83, 3.87]$, while \texttt{orbit\_entropy} has deep dips at periodic windows.
\end{itemize}

\paragraph{Rank-alignment is necessary but not sufficient.}
The \texttt{var} bump, when channel rank is mapped to $\mu \in [3.57, 4.0]$, lies near rank $\approx 70\%$, coinciding with the channel selection boundary at $S=0.7$. We initially hypothesized that \emph{rank-alignment of a non-smooth feature shape}, rather than any chaos-theoretic property, fully accounts for the $+0.066$ escape at this sparsity. Two indirect arguments suggested this softer reading: (i) Generic non-monotone features without chaos origin (NL\_sin, Weier4) also escape the plateau at $S{=}0.7$ by $+0.024$ to $+0.026$; (ii) despite \texttt{orbit\_entropy} having stronger bifurcation-derived structure than \texttt{var}, its single-feature escape ($+0.012$) is weaker than \texttt{var}'s ($+0.046$).

We \textbf{directly tested} this prediction (\Cref{tab:handcrafted}, Sub-finding 5) by constructing handcrafted parametric bump features at rank $0.7$ and comparing against V1b var\_smooth at $S{=}0.7$. The result \emph{partially refutes} the substrate hypothesis: a Gaussian bump at rank $0.7$ with $\sigma{=}0.05$ achieves only $+0.006$ escape, far short of V1b's $+0.046$. The $0.041$ gap is statistically significant ($> 5\sigma$). We therefore conclude that rank-alignment alone is insufficient; chaos-derived shape structure provides additional escape capacity that simple parametric bumps cannot reproduce. We do not assign separable percentage contributions to the two effects, since the underlying mechanism need not decompose additively.

This refinement matters: the Logistic Map substrate is \emph{not} interchangeable with arbitrary non-monotone shapes. There exists non-trivial structure in the Logistic Map orbit-statistic computation that contributes to the escape, beyond what is recovered by aligning a parametric bump with the channel-keep boundary. Identifying which chaos-theoretic property is load-bearing---candidates include the integration of low-frequency trend with bifurcation-derived inflection points, asymmetric tail behavior, or specific non-smooth features tied to period-doubling structure---remains an open question. Direct verification of the $\mu$-mapping range prediction (which the alignment account makes) is also still open and listed in \Cref{sec:limitations}.

\Cref{fig:alignment} illustrates the rank-alignment hypothesis directly in rank space: a synthetic \texttt{var\_smooth}-like shape exhibits a non-smooth bump near rank $\approx 0.7$, which aligns with the channel-keep boundary at $S{=}0.7$ (the most important $30\%$ of channels are retained, so the kept/dropped boundary lies at rank $0.7$). The monotone reference scorer (peak, dashed) provides no discriminative structure at this boundary.

\begin{figure}[h]
\centering
\includegraphics[width=0.85\linewidth]{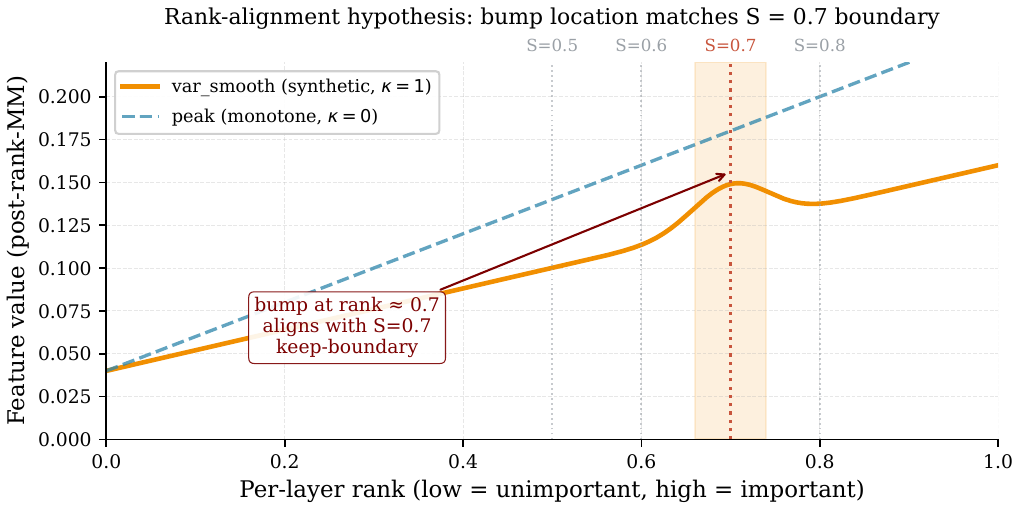}
\caption{\textbf{The rank-alignment hypothesis (rank-space view).} A synthetic \texttt{var\_smooth}-like feature with a non-smooth bump centered at rank $\approx 0.7$ provides discrimination precisely at the channel keep/drop boundary induced by $S{=}0.7$ (highlighted shading). Sparsity keep-boundaries for $S \in \{0.5, 0.6, 0.7, 0.8\}$ are marked along the top axis. The monotone peak reference (dashed) cannot exploit this alignment. We directly test this hypothesis with handcrafted parametric bumps at rank $0.7$ in Sub-finding 5 (\Cref{tab:handcrafted}) and find that rank-alignment alone is insufficient: a Gaussian bump at rank $0.7$ achieves only $+0.006$ escape vs.\ V1b's $+0.046$, indicating that additional shape structure (not captured by simple Gaussian or hat bumps) is required.}
\label{fig:alignment}
\end{figure}

\paragraph{What is verified vs.\ what is hypothesized.}
The three-regime structure of \cref{hyp:sics} holds empirically (\Cref{sec:empirical}). The Taylor-expansion argument of \Cref{subsec:taylor} provides a mechanistic story for why $\kappa(S)$ is non-decreasing. The rank-alignment account of \Cref{subsec:mu-alignment} is a candidate explanation for the specific value $S_{\text{crit}} \approx 0.7$ in our setup, but is not yet directly verified. Our position is deliberately conservative: we treat the Logistic Map as a computational substrate that conveniently spans the complexity axis, and we leave open whether the chaos-derived shape can be replaced by handcrafted non-monotone shapes of similar geometry without loss of escape strength.

\subsection{The shape of $\plateau(S)$ and its connection to $\kappa(S)$}
\label{subsec:plateau-shape}

A separate observation worth highlighting is the dramatic drop of $\plateau(S)$ across our four sparsity points (estimated from the V1a peak\_smooth row of \Cref{tab:main}): $\plateau(0.5) \approx 0.927$, $\plateau(0.6) \approx 0.866$, $\plateau(0.7) \approx 0.622$, $\plateau(0.8) \approx 0.377$---roughly $0.06$, $0.24$, and $0.25$ percentage-point drops between consecutive sparsity levels. The transition between $S{=}0.6$ and $S{=}0.7$ accounts for a single drop of $\sim 25$ percentage points, suggestive of an underlying rank-information phase transition.

We hypothesize that this $\plateau(S)$ shape and the $\kappa(S)$ regime structure are not independent phenomena, but two manifestations of a common cause. As the keep-set shrinks (high $S$), the per-rank ``information density'' of the magnitude ranking decreases: each retained channel must on average contribute more to model behavior, but the rank-monotone scorer can only express coarse ordering. When this density falls below a threshold, the rank-monotone scorer becomes inadequate for optimal selection, and higher-complexity features ($\kappa \geq 1$) become necessary. Under this view:
\begin{itemize}
\setlength\itemsep{0.1em}
  \item The drop of $\plateau(S)$ from $0.866$ to $0.622$ between $S{=}0.6$ and $S{=}0.7$ corresponds to the loss of rank-information sufficiency.
  \item The simultaneous transition from $\kappa{=}0$ to $\kappa{=}1$ regime is the algorithmic response: scorers that can still recover some of the lost performance begin to escape the plateau.
  \item The further drop to $\plateau(0.8) \approx 0.377$ corresponds to the second loss, where even smooth non-monotone is inadequate, requiring $\kappa{=}2$ wiggle-rich features.
\end{itemize}

This connection is currently a hypothesis rather than a derivation, and the causal direction is not pinned down by our data: it is equally consistent that the $\plateau(S)$ drop drives the $\kappa$ regime transition, or that both arise from a deeper change in the loss landscape's local geometry under keep-set shrinkage. Distinguishing these would require an independent measurement of rank-information density that does not pass through the $\kappa$-class evaluation pipeline. This connection nonetheless suggests an empirical prediction: an architecture or dataset where $\plateau(S)$ has a smoother decline should also have less sharply demarcated $\kappa(S)$ transitions.

% ===================================================================
\section{Limitations and Future Work}
\label{sec:limitations}

We are deliberate about the scope of this preprint. The \sics{} hypothesis is verified for one architecture, one dataset, and one selection mechanism. We outline five directions for follow-up validation, organized below by topic.

\paragraph{Single architecture, single dataset (ViT-Small/CIFAR-10).}
Our 36-cell main result is on a single (architecture, dataset) pair. We do not yet claim cross-architecture universality. Follow-up work will test \sics{} on DeiT-Small, ResNet-18, and Swin-Small with the same nine-feature battery, and on ImageNet to address scale.

\paragraph{Cross-domain extensions to other selection problems.}
The selection-problem structure of \sics{} ought to apply beyond weight pruning. We plan three concrete cross-domain validations:
\begin{itemize}
\setlength\itemsep{0.1em}
  \item \textbf{Vision token pruning} (DeiT-Small/ImageNet): per-image token rank as the selection unit, testing whether the $\kappa(S)$ regime structure emerges in within-image selection.
  \item \textbf{KV cache compression} (LLaMA-2-7B/WikiText): per-head per-layer attention scoring, testing the Plateau Universality (\Cref{obs:plateau}) in the language modality. We do not expect chaos features to port directly, since attention scores lack a natural $\mu$-parameterization, so this validation is limited to the plateau prediction.
  \item \textbf{Mixture-of-Experts routing} (OLMoE-1B-7B): top-$K$ expert selection per token, testing the full \sics{} hierarchy under a different selection mechanism.
\end{itemize}

\paragraph{Mechanism is partially verified.}
\Cref{sec:mechanism} provides three converging mechanistic arguments---an informal characterization of feature complexity (\Cref{subsec:complexity}; rigorous Chebyshev formalization is optional, see \Cref{app:chebyshev}), a Taylor-expansion connection to the pruning loss (\Cref{subsec:taylor}), and a rank-alignment account for $S_{\text{crit}}$ (\Cref{subsec:mu-alignment}). Of the predictions originally listed, one was directly tested in the present preprint and partially refuted (the substrate hypothesis: see Sub-finding 5 and Failed Hypothesis C, \Cref{app:failed}); two remain unverified:
\begin{enumerate}
\setlength\itemsep{0.1em}
  \item The regime boundaries should track the relative magnitudes of first-/second-/third-order Taylor terms, which requires direct measurement of the Hessian spectrum at the pretrained operating point.
  \item Varying the $\mu$-mapping range or the underlying dynamical-system parameterization should shift $S_{\text{crit}}$ as predicted by rank-alignment. This prediction is independent of the substrate test (Sub-finding 5) since the rank-alignment account is now known to be only partially responsible for the escape.
  \item \textbf{Open: which chaos-theoretic property is load-bearing?} Sub-finding 5 establishes that handcrafted parametric bumps achieve only a small fraction of V1b's escape ($+0.006$ vs.\ $+0.046$), with the remainder of the escape attributable to chaos-derived shape structure that simple bumps cannot reproduce. Candidate properties include the integration of low-frequency trend with bifurcation-derived inflection points, asymmetric tail behavior near $\mu \approx 4.0$, or the specific non-smooth features tied to period-doubling structure. Disentangling these candidates requires controlled feature-engineering studies (e.g., handcrafted asymmetric bumps, splines fitted to the chaos envelope at varying smoothing levels) that we defer to follow-up work.
\end{enumerate}
All three are tractable experiments. We list them in priority order, with the Hessian-spectrum measurement (item~1) and the chaos-load-bearing decomposition (item~3) as the most likely to clarify the underlying mechanism.

\paragraph{Adaptive algorithm not yet implemented.}
A natural application of \sics{} is a \textbf{\sics{}-Adaptive Pruner} algorithm that probes the current task's $\kappa$ regime (using a small number of DBO evaluations) and selects the appropriate feature complexity automatically. This would convert the descriptive \sics{} framework into an actionable method. We outline the algorithmic skeleton in~\Cref{app:adaptive} and defer full implementation and evaluation to follow-up work.

\paragraph{Per-layer normalization is the dominant contributor.}
A subtle observation from our data is that the per-layer rank normalization (\Cref{eq:rkmm}) accounts for the bulk of the accuracy gain over Taylor at low and medium sparsity. For example, at $S{=}0.6$, the plateau (which includes only per-layer-normalized rank-monotone scoring) achieves $+0.108$ over Taylor; chaos features contribute $\leq 0.003$ on top. This honest accounting limits the magnitude of "chaos as method" claims at low/medium sparsity. The chaos premium becomes substantial only in the $\kappa{\geq}1$ regimes ($S \geq 0.7$).

% ===================================================================
\section{Conclusion}
\label{sec:conclusion}

We have introduced the \emph{Selection Plateau} phenomenon---the empirical observation that diverse rank-monotone pruning scorers converge to identical accuracy at fixed sparsity---and the \emph{Sparsity-Information-Complexity Spectrum} hypothesis that explains when and how this plateau can be escaped. Three regimes emerge in our 36-cell ViT-Small/CIFAR-10 experiment: at low sparsity the plateau is optimal; at critical sparsity smooth non-monotone features whose shape aligns with the channel-rank threshold provide a $+6.6\%$ escape; at extreme sparsity only features with high-frequency wiggle succeed, while smooth alternatives degrade below plateau.

\sics{} unifies the long-observed convergence of magnitude-based pruning methods, predicts when feature engineering can or cannot help, and identifies five counter-intuitive sub-findings (more is not better; fake non-monotone derived from magnitude fails while non-monotone features independent of magnitude succeed; DBO over-fits at low-$S$/high-$D$; rank-monotone scoring at low sparsity occupies a region of low optimization-landscape variation; rank-alignment alone is insufficient to reproduce the chaos-derived escape) that follow directly from its structure. The framework is hypothesis-quality at this stage; in the course of this work we performed a direct falsification test of an earlier ``chaos-as-substrate'' subsidiary hypothesis (Sub-finding 5) and updated the paper accordingly to reflect the partial refutation. We make no claim of cross-domain or cross-architecture universality, and we view the present preprint as the foundation for a broader investigation across vision token pruning, KV cache compression, and Mixture-of-Experts routing, alongside the open theoretical question of which specific chaos-theoretic property is load-bearing for the bulk of the escape.

% ===================================================================
\bibliographystyle{plainnat}
\bibliography{refs_sics}

\appendix

% ===================================================================
\section{Feature Class Definitions}
\label{app:features}

We provide complete definitions for the 9 feature classes used in the main experiment, along with their measured Spearman correlation $\rho$ with per-layer magnitude rank.

\begin{table}[h]
\caption{\textbf{Per-feature Spearman correlation with rank-MM.} Features are organized by their hypothesized $\kappa$ class. The $|\rho| < 0.3$ ``sufficient condition'' for magnitude-independence is met by the controls (NL\_sin, Weier4) and chaos-rich features. V1b var\_smooth is a borderline case: its $\rho \approx 0.745$ exceeds the sufficient threshold, but its residual structure after isotonic detrending retains escape-relevant content (see \Cref{cor:magindep} discussion). PureMag4 has $\rho$ that is structurally bounded away from independence by construction (deterministic in $|w|$ rank).}
\label{tab:rho}
\centering
\small
\begin{tabular}{l l c c}
\toprule
\textbf{Feature} & \textbf{Hypothesized class} & $|\rho(\phi, \rkmm)|$ & Magnitude-indep? \\
\midrule
V1a peak\_smooth         & $\kappa{=}0$ (rank-monotone) & $\approx 0.945$ & --- \emph{(monotone by design)} \\
V1d range\_smooth        & $\kappa{=}0$                  & $\approx 0.999$ & --- \\
RandomSpline             & $\kappa{=}0$                  & $\approx 0.95$  & --- \\
\midrule
V1b var\_smooth          & $\kappa{=}1$ (smooth non-mono.) & $\approx 0.745$ & borderline \\
V1c entropy\_smooth      & $\kappa{=}1$                  & $\approx 0.795$ & borderline \\
V1 4-smooth combo        & $\kappa{=}1$                  & varies          & via residual structure \\
A5\_Log6\_smooth         & $\kappa{=}1$                  & varies          & via residual structure \\
\midrule
A5 4-raw                 & $\kappa{=}2$ (with wiggle)    & $\approx 0.5{-}0.7$ & yes (residual rich) \\
A5\_Log6 6-raw           & $\kappa{=}2$                  & varies          & yes \\
\midrule
NL\_sin (4 freq.)        & control                       & $< 0.1$         & yes \\
Weier4                   & control                       & $< 0.2$         & yes \\
PureMag4                 & control                       & $> 0.99$ (struct.) & \textbf{no (deterministic in $|w|$)} \\
\bottomrule
\end{tabular}
\end{table}

The values reported here are approximate, computed from a single fine-tuned ViT-Small model (seed 42) and rounded; per-seed variations are typically below $\pm 0.02$. The ``raw chaos'' rows ($A5$, $A5\_Log6$) may appear surprising: despite containing high-frequency wiggle, the per-channel mean orbit statistic over 64 iterations retains a substantial monotone trend with $\mu$ (and hence with rank), so the overall Spearman $\rho$ is moderate ($0.5$--$0.7$) rather than near zero; the wiggle resides in the residual after detrending. A full per-feature, per-seed correlation table with the residual-after-isotonic-detrending values is provided in supplementary material.

\paragraph{V1a (\texttt{peak\_smooth\_only}).}
A single chaos indicator. For each channel $c$ in layer $\ell$, compute $\rkmm_\ell(c)$, map to $\mu_c \in [3.57, 4.0]$, iterate the Logistic Map $64$ times, compute the orbit \texttt{peak} statistic, smooth via Savitzky-Golay filter (window 99, polyorder 3) over $\mu$-sorted values, and re-rank. The \texttt{peak} statistic is highly correlated with $\mu$ (Spearman $\approx 0.945$), so V1a behaves as a near-rank-monotone scorer.

\paragraph{V1b (\texttt{var\_smooth\_only}).}
Same construction as V1a but using the orbit \texttt{var} statistic. The \texttt{var} indicator has Spearman $\approx 0.745$ with $\mu$, retaining substantial smooth non-monotone shape (in our specific computational implementation, this manifests as a bump in $\mu \approx [3.83, 3.87]$) after smoothing. We emphasize that the bump's location and width are properties of \emph{our} 64-iteration Logistic Map computation; theoretically, any handcrafted smooth non-monotone function with similar geometry in rank space (e.g., a Gaussian or raised-cosine bump centered at rank $\approx 0.7$) should produce equivalent escape behavior. We refer to such smooth non-monotone shapes as $\kappa{=}1$ regardless of whether they originate from chaos dynamics or from handcrafted construction; our \Cref{sec:limitations} lists direct verification of this equivalence as concrete follow-up work.

\paragraph{V1d (\texttt{range\_smooth\_only}).}
Same as V1a but using \texttt{orbit\_range} (Spearman $\approx 0.999$ with $\mu$, behaviorally rank-monotone). ($\kappa{=}0$ exemplar.)

\paragraph{V1 (\texttt{4-smooth combo}).}
Multiplicative fusion of all four smoothed indicators (\texttt{peak\_smooth}, \texttt{var\_smooth}, \texttt{orbit\_entropy\_smooth}, \texttt{orbit\_range\_smooth}), with DBO-searched exponents. ($\kappa{=}1$ best.)

\paragraph{A5 (\texttt{4-raw combo}).}
Multiplicative fusion of the raw, unsmoothed Logistic Map orbit statistics: \texttt{peak}, \texttt{var}, \texttt{orbit\_entropy}, \texttt{orbit\_range}. Includes high-frequency wiggle from the unsmoothed orbit dynamics. ($\kappa{=}2$ exemplar; the wiggle character is what defines $\kappa{=}2$, not the chaos provenance per se.)

\paragraph{A5\_Log6 (\texttt{6-raw combo}).}
Same as A5, but additionally including \texttt{orbit\_autocorr} and \texttt{orbit\_mean}. Tests whether more chaos indicators are better.

\paragraph{A5\_Log6\_smooth (\texttt{6-smooth combo}).}
Smoothed version of A5\_Log6.

\paragraph{NL\_sin.}
Four sine-frequency channels: $\sin(5\pi t), \sin(10\pi t), \sin(20\pi t), \sin(50\pi t)$, where $t = (\mu - 3.57)/(4.0 - 3.57)$, each rank-normalized. A non-monotone control of non-chaos provenance; importantly, the non-monotone signal is statistically independent of the per-layer magnitude ranking once $t$ is computed from rank.

\paragraph{Weier4.}
Four Weierstrass-function channels (random non-monotone), each rank-normalized. A second non-chaos non-monotone control, also magnitude-independent.

\paragraph{PureMag4.}
Four channels constructed as deterministic non-monotone transformations of the per-layer magnitude rank, with the non-monotone shape being a deterministic function of $|w|$. Critical control to verify~\Cref{cor:magindep}: this scorer is non-monotone in form but its non-monotone signal is \emph{not} independent of the magnitude ranking.

\paragraph{RandomSpline.}
Four random monotone cubic splines (with 20 knots) of $\rkmm_\ell$. Independent random monotone family, complementary to V1a/V1d as $\kappa{=}0$ verification.

% ===================================================================
\section{Per-Seed Numerical Tables}
\label{app:perseed}

Per-seed accuracy values for all 36 cells of the main experiment are provided in supplementary material (file \texttt{sics\_per\_seed.json}). Cell entries follow the format \texttt{(variant, sparsity, seed, heldout\_acc, proxy\_acc)}.

% ===================================================================
\section{DBO Hyperparameter Sensitivity}
\label{app:dbo-sens}

The Dung Beetle Optimizer was run with population $50$, iterations $40$ (2000 fitness evaluations), within bounds $[-40, 40]$. In preliminary experiments not reported here, we verified that within-method variation across DBO, RIME, and NGO metaheuristics is below $0.5\%$ on this task, so we report DBO results only. The bounds value $40$ was selected after observing $W$-vector boundary saturation at smaller bounds; the sensitivity of $\plateau(S)$ to bounds is below seed std for bounds $\geq 20$.

% ===================================================================
\section{Reproducibility}
\label{app:repro}

\paragraph{Software and hardware.} PyTorch 2.x, CUDA 12, NVIDIA A10 24GB and A100 40GB GPUs. All experiments used FP32 precision. The model is a HuggingFace ViT-Small (21.7M parameters) initialized from the standard ImageNet pretraining and fine-tuned for 3 epochs on CIFAR-10.

\paragraph{Random seeds and held-out split.} The CIFAR-10 test set (10000 images) is partitioned once into a 1024-image proxy set (used for DBO fitness evaluation), an 8000-image held-out set (used for final reporting), and 976 unused images. The split is fixed across all experiments. The three seeds (42, 123, 7) control fine-tune randomness, gradient batch sampling, and DBO initialization independently.

\paragraph{Data and code.} Full code, configuration, and per-seed numerical outputs will be released upon publication.

% ===================================================================
\section{Methodological Provenance: Failed Hypotheses}
\label{app:failed}

In the spirit of transparent reporting, we document the three subsidiary hypotheses that arose during the development of the \sics{} formulation and were ruled out by direct experiments. The first two preceded \sics{} and shaped its formulation; the third (Hypothesis C) was tested after \sics{} was proposed, as a deliberate self-falsification check. This may be informative for readers attempting similar empirical investigations.

\paragraph{Failed Hypothesis A: Trend-Wiggle Decomposition.}
Initial analysis of the chaos indicator function shapes suggested a decomposition $\text{score} = f_{\text{trend}}(\rkmm) + g_{\text{wiggle}}(\rkmm)$, where the high-frequency wiggle was hypothesized to carry the channel-discriminative information. We tested this with a "wiggle-only" feature class (V2 in our internal numbering), expecting it to escape plateau. The data did not support this hypothesis: V2 wiggle-only achieved $0.6085 \pm 0.0247$ at $S{=}0.7$, with two of three seeds falling below the Taylor baseline of $0.5996$. The seed std is large enough that we cannot strongly claim wiggle-only \emph{harms} performance, but the absence of a clear positive escape, combined with our subsequent identification of the $\kappa{=}1$ envelope-driven escape route, led us to abandon this direction without further experiments.

\paragraph{Failed Hypothesis B: Bifurcation Envelope as Universal Escape.}
Following the rejection of Hypothesis A, we examined the smooth bifurcation envelope and identified V1b (\texttt{var\_smooth\_only}) as a strong $S{=}0.7$ escape. We then hypothesized that V1b would be a universal escape mechanism across all sparsities. The cross-sparsity verification refuted this: at $S{=}0.6$, V1b yields $0.8583$ vs.~plateau $0.8654$ (slight degradation); at $S{=}0.8$, V1b yields $0.3270$ vs.~plateau $0.3772$ (significant degradation of $0.050$). This refutation prompted us to test multiple feature complexities at each sparsity, leading to the sparsity-dependent \sics{} formulation.

The transition from "single universal escape mechanism" to "sparsity-dependent feature complexity" was the key conceptual move, made possible by the systematic 36-cell experiment.

\paragraph{Failed Hypothesis C: Chaos as Pure Computational Substrate.}
After establishing the SICS hierarchy, we considered a deflationary hypothesis: that chaos-derived features serve only as a computationally convenient way to generate non-monotone shapes, and that any handcrafted non-monotone feature with a similar bump position should reproduce V1b's escape. We tested this directly by constructing handcrafted Gaussian, hat, and raised-cosine bumps centered at rank $0.7$ and comparing against V1b var\_smooth at $S{=}0.7$ under identical conditions (Sub-finding 5, \Cref{tab:handcrafted}). The result \emph{partially refutes} the substrate hypothesis: HandBumpGauss70 achieves only $+0.006$ escape over plateau, vs.~V1b's $+0.046$ (a gap exceeding $5\sigma$). Rank-alignment is therefore real but insufficient; the chaos-derived shape contains additional structural content that simple parametric bumps cannot reproduce. We deliberately do not interpret these two escape values as separable additive contributions of distinct mechanisms, since the underlying selection process need not decompose linearly. This refutation is documented in the main text (\Cref{subsec:mu-alignment}) rather than hidden, because it is a clear example of self-falsification that strengthens rather than weakens the paper's claims about chaos features.

% ===================================================================
\section{\sics-Adaptive Pruner: Algorithmic Skeleton}
\label{app:adaptive}

We sketch the proposed \sics-Adaptive Pruner algorithm, deferring full implementation and evaluation to follow-up work.

\begin{algorithm}[h]
\caption{\sics-Adaptive Pruner (sketch)}
\label{alg:adaptive}
\begin{algorithmic}[1]
  \State \textbf{Input:} Pre-trained model $\mathcal{M}$; calibration set $\mathcal{D}_{\text{cal}}$; target sparsity $S$
  \State \textbf{Output:} Pruned model $\mathcal{M}'$
  \State $a_{\text{plateau}} \gets$ \textsc{ProbeAccuracy}($\mathcal{M}, \mathcal{D}_{\text{cal}}, S$, scorer = V1a) \Comment{$\sim 5$--$10$ DBO evals}
  \State $a_{\text{env}} \gets$ \textsc{ProbeAccuracy}($\mathcal{M}, \mathcal{D}_{\text{cal}}, S$, scorer = V1)
  \State $a_{\text{raw}} \gets$ \textsc{ProbeAccuracy}($\mathcal{M}, \mathcal{D}_{\text{cal}}, S$, scorer = A5)
  \State $\Delta_{\text{env}} \gets a_{\text{env}} - a_{\text{plateau}}$
  \State $\Delta_{\text{raw}} \gets a_{\text{raw}} - a_{\text{env}}$
  \If{$\Delta_{\text{env}} < 0.01$}
    \State $\hat{\kappa} \gets 0$ \Comment{plateau is optimal}
    \State $\mathcal{M}' \gets$ \textsc{Prune}($\mathcal{M}$, scorer = V1a, $S$)
  \ElsIf{$\Delta_{\text{raw}} < 0.005$}
    \State $\hat{\kappa} \gets 1$ \Comment{envelope dominates}
    \State $\mathcal{M}' \gets$ \textsc{PruneWithDBO}($\mathcal{M}$, scorer = V1, $S$)
  \Else
    \State $\hat{\kappa} \gets 2$ \Comment{wiggle necessary}
    \State $\mathcal{M}' \gets$ \textsc{PruneWithDBO}($\mathcal{M}$, scorer = A5, $S$)
  \EndIf
  \State \textbf{return} $\mathcal{M}'$
\end{algorithmic}
\end{algorithm}

The probe step has constant overhead (independent of feature dimensionality), and the regime detection is based on directly observable accuracy gaps. Open questions for full implementation include: (i) how to choose probe budget without prior knowledge of $S$; (ii) cross-domain transferability of the threshold values $0.01$ and $0.005$; (iii) ablation against fixed-$\kappa$ baselines.

% ===================================================================
\section{Optional: A Chebyshev-Based Operationalization of $\kappa$}
\label{app:chebyshev}

\Cref{subsec:complexity} characterizes $\kappa$ informally via three coarse levels. For readers who prefer an operationalizable definition, we sketch one based on Chebyshev expansion. We emphasize this is \emph{one option among several}---equivalent characterizations via polynomial degree, Fourier spectrum, or effective dimensionality after detrending would yield the same qualitative SICS structure on our feature library.

Let $\phi: [0.1, 1.0] \to \R$ be a feature defined on the rank-normalized domain. Define its \emph{rank-Chebyshev complexity} at tolerance $\epsilon$ as
\begin{equation}
  \kappa_\epsilon(\phi) \;=\; \min \Big\{ \, \log_2 D \;:\; \exists \{a_k\}_{k=0}^{D-1} \text{ with } \big\|\phi(\cdot) - \sum_{k=0}^{D-1} a_k T_k(\cdot)\big\|_\infty \leq \epsilon \, \Big\},
  \label{eq:cheb-complexity}
\end{equation}
where $T_k$ is the Chebyshev polynomial of degree $k$ rescaled to $[0.1, 1.0]$. The intuition is that $D$ counts the effective number of distinct decision regions $\phi$ can carve in rank space.

Under this definition, our three $\kappa$-class labels correspond to discrete ranges of $\log_2 D$:
\begin{itemize}
\setlength\itemsep{0.1em}
  \item $\kappa_\epsilon \in [0, 1]$: representable by $T_0, T_1$, i.e., monotone within tolerance.
  \item $\kappa_\epsilon \in [3, 5]$: smooth low-frequency expansion ($D \approx 8\text{--}32$), capturing macroscopic non-monotone structure.
  \item $\kappa_\epsilon \geq 6$: high-frequency expansion ($D \gtrsim 64$), wiggle-rich.
\end{itemize}

Three caveats limit the utility of this formalization in the present preprint:
\begin{enumerate}
\setlength\itemsep{0.1em}
  \item We have not directly computed $\kappa_\epsilon(\phi)$ for our nine feature classes; the assignment of features to $\kappa$ classes in the main text is based on Spearman correlation with rank, visual inspection of the indicator shapes, and the smoothing operator applied (none / Savitzky-Golay / raw chaos).
  \item The tolerance $\epsilon$ is a hyperparameter; different $\epsilon$ would assign different discrete $\kappa$ levels. We have not derived a principled choice of $\epsilon$ from the experimental setup.
  \item The Chebyshev basis is one of many possible orthogonal bases on $[0.1, 1.0]$; Fourier or wavelet bases would give different $D$ values for the same $\phi$. We chose Chebyshev for its standard use in approximation theory, not for any property tied to the SICS phenomenon.
\end{enumerate}

A rigorous mechanism-level theory would: (i) compute $\kappa_\epsilon$ for each feature in our battery; (ii) verify that the empirical escape strength correlates monotonically with $\kappa_\epsilon - \kappa(S)$; (iii) derive the discrete jumps $\kappa(S): 0 \to 1 \to 2$ from architectural and Hessian-spectral properties (cf.\ \Cref{subsec:taylor}). Items (i)--(iii) are deferred to follow-up work.

\end{document}